\pdfoutput=1

\documentclass[11pt]{article}

\PassOptionsToPackage{dvipsnames}{xcolor}

\usepackage[final]{acl}

\usepackage{times}
\usepackage{latexsym}
\usepackage{hyperref}
\usepackage{url}
\usepackage{pifont}
\usepackage{amsmath}
\usepackage{pifont}   
\usepackage{colortbl} 
\usepackage{makecell} 

\newsavebox{\tablebox} 
\usepackage{amssymb}
\usepackage{amsthm}
\usepackage{algorithmic}
\usepackage{algorithm}
\usepackage{amsmath, amssymb}
\usepackage{subcaption}
\usepackage{graphicx}
\usepackage{graphicx}
\usepackage{booktabs}
\usepackage{multirow}

\usepackage{graphicx}
\usepackage{xcolor,colortbl}
\usepackage{xcolor} 
\usepackage{booktabs}
\usepackage{makecell}
\usepackage{graphicx}
\usepackage{multirow}  
\usepackage{booktabs}  
\usepackage{graphicx}  
\usepackage{xcolor}    
\usepackage{colortbl}  
\usepackage[breakable,skins]{tcolorbox} 
\usepackage{xspace}
\usepackage{lipsum}
\usepackage{tcolorbox}
\newtheorem{definition}{Definition}
\usepackage[utf8]{inputenc}
\usepackage[dvipsnames]{xcolor}  
\usepackage{tcolorbox}
\usepackage{listings}
\tcbuselibrary{skins, breakable}
\newtheorem{theorem}{Theorem}
\newtheorem{lemma}{Lemma}

\newtheorem{assumption}{Assumption}

\definecolor{mygrey}{RGB}{200, 200, 200}

\definecolor{red3}{RGB}{205,55,55}
\definecolor{primaryblue}{RGB}{102,126,234}
\definecolor{secondaryblue}{RGB}{79,172,254}
\definecolor{lightgray}{RGB}{245,245,245}
\definecolor{darkgray}{RGB}{64,64,64}
\definecolor{promptboxblue}{RGB}{51, 122, 183}
\usepackage{enumitem}
\usepackage{booktabs}
\tcbuselibrary{skins,breakable}

\newtcolorbox{exmp}[3][]{
    enhanced,
    breakable,
    colback=lightgray,
    colframe=darkgray,
    colbacktitle=darkgray,
    coltitle=white,
    fonttitle=\bfseries,
    title={#2},
    label={#3},
    top=0.5mm,
    bottom=0.5mm,
    left=2mm,
    right=2mm,
    boxsep=0.5mm,
    toptitle=1mm,
    bottomtitle=1mm,
    attach boxed title to top left={yshift=-2mm,xshift=3mm},
    boxed title style={
        enhanced,
        size=small,
        colback=darkgray,
        colframe=darkgray,
    },
    sharp corners,
    #1
}

\usepackage{xcolor}
\usepackage{booktabs}
\usepackage{multirow}
\usepackage{makecell}
\usepackage{graphicx}  
\usepackage{bbding}    
\usepackage{colortbl}

\definecolor{darksalmon}{rgb}{0.91, 0.59, 0.48}

\newcommand{\blue}[1]{$_{\color{BlueGreen}\downarrow #1}$}
\newcommand{\red}[1]{$_{\color{RedOrange}\uparrow #1}$}
\newcommand{\hlfirst}[1]{\colorbox[HTML]{EAE1EF}{#1}}
\newcommand{\hlsecond}[1]{\colorbox[HTML]{F8F5FA}{#1}}

\definecolor{custom}{HTML}{8583A9}
\newcommand{\framework}{\textcolor{custom}{\textsc{CoAct}}}
\usepackage[T1]{fontenc}

\usepackage[utf8]{inputenc}

\usepackage{microtype}

\usepackage{inconsolata}

\usepackage{graphicx}

\title{\framework: Co-Active LLM Preference Learning with Human-AI Synergy}

\author{
\textbf{Ruiyao Xu}$^{\star}$ \quad
\textbf{Mihir Parmar}$^{\diamond}$ \quad
\textbf{Tiankai Yang}$^{\clubsuit}$ \quad
\textbf{Zhengyu Hu}$^{\heartsuit}$ \\
\textbf{Yue Zhao}$^{\clubsuit}$ \quad
\textbf{Kaize Ding}$^{\star}$ \\
\\[-0.8em]
$^{\star}$Northwestern University \quad
$^{\diamond}$Google \\ \quad
$^{\clubsuit}$University of Southern California \quad
$^{\heartsuit}$University of Washington 
}

\begin{document}
\maketitle
\begin{abstract}
Learning from preference-based feedback has become an effective approach for aligning LLMs across diverse tasks. However, high-quality human-annotated preference data remains expensive and scarce. Existing methods address this challenge through either self-rewarding, which scales by using purely AI-generated labels but risks unreliability, or active learning, which ensures quality through oracle annotation but cannot fully leverage unlabeled data. In this paper, we present \framework, a novel framework that synergistically combines self-rewarding and active learning through strategic human-AI collaboration. \framework\xspace leverages self-consistency to identify both reliable self-labeled data and samples that are requiring oracle verification. Additionally, oracle feedback guides the model to generate new instructions within its solvable capability. Evaluated on three reasoning benchmarks across two model families, \framework~achieves average improvements of +13.25\% on \texttt{GSM8K}, +8.19\% on \texttt{MATH}, and +13.16\% on \texttt{WebInstruct}, consistently outperforming all baselines. \footnote{Our code is available at \url{https://github.com/rux001/CoAct}.}
\end{abstract}

\section{Introduction}
\label{sec:introduction}
Preference alignment has demonstrated remarkable performance across a wide array of tasks, including instruction following, question answering, math reasoning, and creative writing~\cite{rafailov2024direct, ouyang2022training,bai2022training, christiano2017deep,hu2024explaining,wang2025polo}. However, the \textit{scarcity} of \textit{high-quality} human-annotated pairwise preference data severely limits the effectiveness and scalability of preference learning methods~\cite{luo-etal-2025-survey, yuan2024self, li2024empowering, lee2023rlaif,lei2026humanllm,hu2026towards}. 
Existing approaches~\citep{huang-etal-2023-large, yuan2024self, shenactive, das2024active} tackle either \textit{data scarcity} or \textit{data quality} in isolation, each making distinct trade-offs. 

\begin{figure}[tbp]
\centering
\includegraphics[width=\columnwidth]{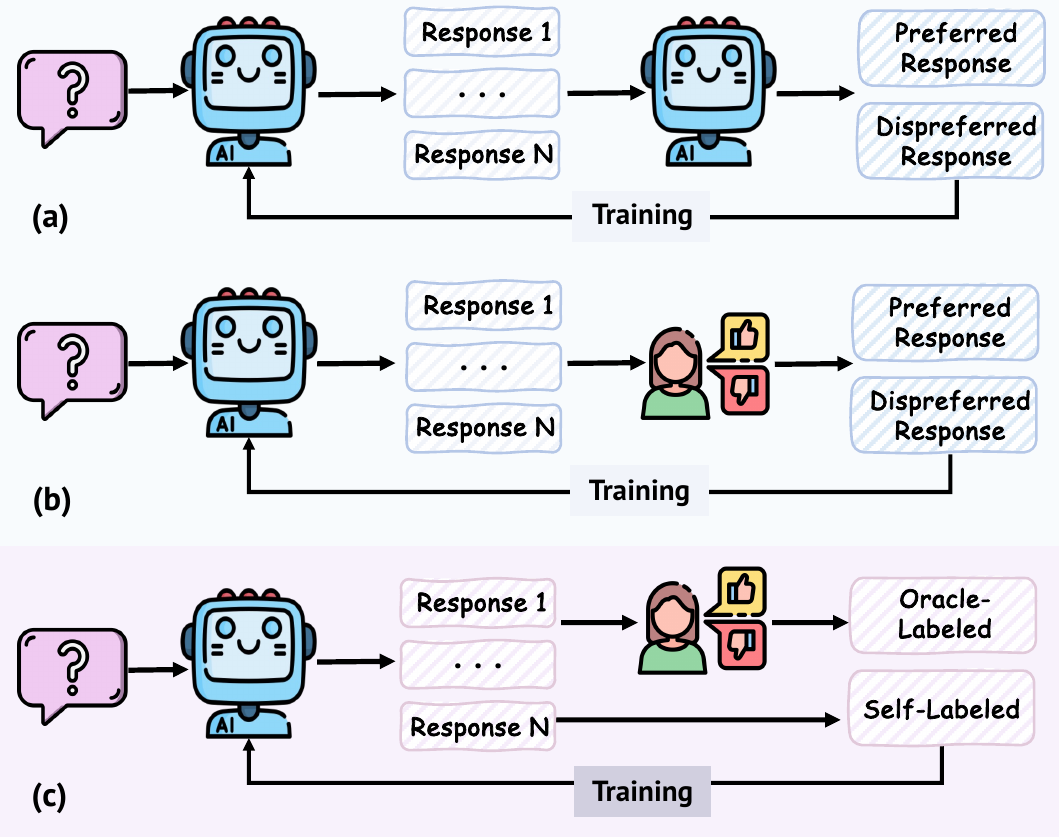}
\caption{\textbf{(a)}~Self-rewarding uses AI self-labeled data to construct preference pairs; \textbf{(b)}~Active preference learning uses human annotation to ensure data quality; \textbf{(c)}~Our framework \framework~combines both approaches through human-AI collaboration.}
\label{fig:framework}
\vspace{-20pt}
\end{figure}

To address the challenge of \textit{data scarcity} in preference learning, recent works explore ``self-rewarding'' pipelines that utilize LLMs themselves to generate new instructions, produce candidate responses, and construct preference pairs through self-evaluation~\cite{huang-etal-2023-large, yuan2024self, chen2024self, rosset2024direct}. By leveraging the model's own judgments to create training data, these self-rewarding frameworks dramatically reduce the need for human annotation or external reward models. While promising in their scalability, these approaches face a critical limitation: without external validation, they are prone to amplifying self-bias errors, where models reinforce their own errors and misconceptions through iterative self-training, potentially diverging from true human preferences~\citep{laidlaw2024correlated, zhang2025no, shafayat2025can, ding2024divide, huang2023large, bansal2023large}. 

On the other hand, active preference learning has been explored to ensure \textit{data quality}~\cite{shenactive, das2024active,lin2025activedpo,muldrew2024active}. These methods incorporate an iterative data acquisition and fine-tuning loop, where at each iteration, the most informative samples are strategically selected from an unlabeled pool for human annotation. However, active preference learning faces its own fundamental limitation: the constrained annotation budget prevents utilization of the vast amounts of remaining unlabeled data, and thus, potentially valuable samples are neglected. These two paradigms present a fundamental dilemma: \textit{How can we achieve better data efficiency for LLM alignment by leveraging the synergy between human and AI?} 

In this paper, we introduce \framework, a novel framework that bridges self-rewarding and active learning for preference alignment through strategic human-AI synergy as shown in Figure~\ref{fig:framework}. Our framework operates through an iterative process: at each round, we generate multiple responses for each unlabeled instruction and construct preference pairs through the idea based on self-consistency~\cite{wang2022self, prasad2024self, jiao2025preference} -- the most consistent response is considered as chosen and the least consistent one is used as rejected. Based on the consistency score of the chosen response, we partition samples into high-consistency and low-consistency subsets. For high-consistency samples, we further identify samples with potential self-consistent errors through k-NN distance metrics~\cite{sun2022knn}. Those samples will be routed for oracle labeling along with low-consistency subset. The remaining high-consistency samples are directly used as self-labeled training data. Remarkably, oracle feedback serves a dual purpose: \ding{192} providing reliable training signals through verified preference pairs, and \ding{193} guiding new instruction generation, where oracle-verified examples serve as in-context demonstrations to generate instructions within the model's solvable capability. Finally, both oracle-labeled and self-labeled preference pairs are combined to update the model using a modified DPO objective that incorporates both the DPO loss and an NLL regularization term~\citep{pang2024iterative}.

In our experiments, we evaluate \framework~using two model families (\texttt{Llama3-8B} and \texttt{Qwen3-4B}) across three reasoning benchmarks (\texttt{GSM8K}, \texttt{MATH}, \texttt{WebInstruct}). Experimental results demonstrate that \framework~achieves substantial performance improvements over baseline methods. Specifically, \framework~achieves average gains of +13.25\% on \texttt{GSM8K}, +8.19\% on \texttt{MATH}, and +13.16\% on \texttt{WebInstruct} across both models after four training iterations. Notably, \framework~outperforms the strongest baseline by 4-8 percentage points at the final iteration. Beyond in-domain performance, \framework~demonstrates strong generalization to out-of-domain benchmarks, consistently achieving the best performance on \texttt{GPQA} and \texttt{MMLU-Pro}. Moreover, we analyze the effectiveness of self-consistency for preference construction, observing strong Pearson correlations with accuracy.

\section{Related Work}
\paragraph{LLM Preference Alignment.}
Preference alignment aims to align LLMs with human preferences across dimensions including safety, helpfulness, factuality, reasoning, and scientific discovery~\citep{askell2021general, ouyang2022training,hu2025population,wang2025survey,wang2026molmemmemoryaugmentedagenticreinforcement}. Reinforcement Learning from Human Feedback (RLHF)~\citep{leike2018scalable, stiennon2020learning} is a prevalent approach that trains a reward model from human preferences and uses reinforcement learning algorithms such as PPO to optimize the language model~\citep{bai2022training, christiano2017deep}. Direct Preference Optimization (DPO)~\citep{rafailov2024direct} has emerged as a more efficient alternative, eliminating the explicit reward model by directly optimizing preference probabilities. Several extensions have since been proposed, including KTO~\citep{ethayarajh2024kto}, GPO~\citep{zhao2023group}, $\Psi$PO~\citep{azar2024general}, and ODPO~\citep{amini2024direct}.

However, high-quality annotated preference data remains limited and expensive to obtain. Recent work leverages LLMs themselves to generate or verify preference data, commonly referred to as RLAIF or self-rewarding~\citep{lee2023rlaif, yuan2024self, chen2024self, prasad2024self, shafayat2025can, wu2025meta,liu2025guardreasoner,hu2025unveiling}. For instance, \citet{yuan2024self} propose self-rewarding language models, where the model acts as its own judge via LLM-as-a-Judge prompting to evaluate self-generated responses and iteratively improve itself. While this approach improves scalability, it introduces the risk of self-bias~\citep{bansal2023large,wang2024human,wang2022self,xu2026gnnasjudge}. Prior works have explored human-AI collaboration for data generation in traditional NLP~\citep{bartolo-etal-2022-models, liu-etal-2022-wanli, wang2024human}. Recent work~\citep{liu2025skywork} explores human-AI collaboration for generating preference data for reward model training. However, their approach relies on a group of strong LLMs to aggregate preference judgments, which remains resource-intensive.

\paragraph{Active LLM Alignment.}
Active learning for LLM alignment seeks to maximize alignment quality while minimizing human annotation costs by strategically selecting queries for labeling. Early work on active preference learning focused on applications in robotics and autonomous systems~\citep{Sadigh2017ActivePL, biyik2018batch,wang2025llm}. In the context of LLMs, recent approaches can be broadly categorized into heuristic methods and theoretically grounded frameworks. Heuristic methods leverage uncertainty-based metrics~\citep{ muldrew2024active, melo2024deep, gleave2022uncertainty,hu2024let} or margin-based selection criteria~\citep{muldrew2024active} to identify high-value samples. From a theoretical perspective, \citet{das2024active} introduce Active Preference Optimization (APO), which frames active selection as a contextual dueling bandit problem and proves near-optimal sample complexity bounds. Similarly, \citet{mehta2023sample} study dueling bandits for preference learning with theoretical guarantees. In contrast, \citet{lin2025activedpo} proposes a selection criterion for non-linear reward functions that directly leverages the LLM itself to parameterize the reward model used for active data selection. Nevertheless, their approach has two key limitations: \ding{192} the selection criterion may be less effective for complex reasoning tasks, and \ding{193} they do not fully exploit the unlabeled data pool.

\section{Preliminary}
\label{sec:preliminary}
In this section, we establish the notation and formalize the problem setting for active preference learning:
\begin{definition}[Active Preference Learning]
Let $\mathcal{D}_U = \{x_j\}_{j=1}^{N_U}$ denote an unlabeled instruction pool and 
$\mathcal{D}_L = \{(x_i, y_i^+, y_i^-)\}_{i=1}^{N_L}$ an initial set of labeled
preference pairs with $y_i^+ \succ y_i^-$. An Active Preference Learning algorithm proceeds in iterations: at each step $t$, given a batch
$\mathcal{B}_t \subset \mathcal{D}_U$ of instructions, generates multiple candidate 
responses per instruction using the current model $\theta_t$, select the top $M$ 
pairs for oracle preference labeling to obtain $\mathcal{D}_\textit{oracle}^{(t)}$ based on an acquisition function. The labeled 
set is augmented as $\mathcal{D}_L \leftarrow \mathcal{D}_L \cup \mathcal{D}_\textit{oracle}^{(t)}$, 
and model parameters are updated via a preference-learning objective.
\end{definition}

We now formalize \textit{Co-Active Preference Learning}, which integrates human and AI supervision:
\begin{definition}[Co-Active Preference Learning]
Building upon APL, Co-Active Preference Learning augments the active learning loop 
with self-generated supervision. At each iteration $t$, given batch $\mathcal{B}_t \subset \mathcal{D}_U$, 
the algorithm constructs self-labeled preference pairs $\mathcal{D}_{\textit{AI}}^{(t)} = \{(x, \tilde{y}^+, \tilde{y}^-) \mid x \in \mathcal{B}_t\}$ 
alongside selecting $M$ pairs for oracle labeling to obtain $\mathcal{D}_{\textit{oracle}}^{(t)}$. The labeled set is augmented as $\mathcal{D}_L \leftarrow \mathcal{D}_L \cup \mathcal{D}_{\textit{oracle}}^{(t)} \cup \mathcal{D}_{\textit{AI}}^{(t)}$, 
and model parameters are updated via a preference-learning objective.
\end{definition}

\section{\framework: Human-AI Co-Active Preference Learning}

\begin{figure*}[t]
   \centering
   \includegraphics[width=1\textwidth]{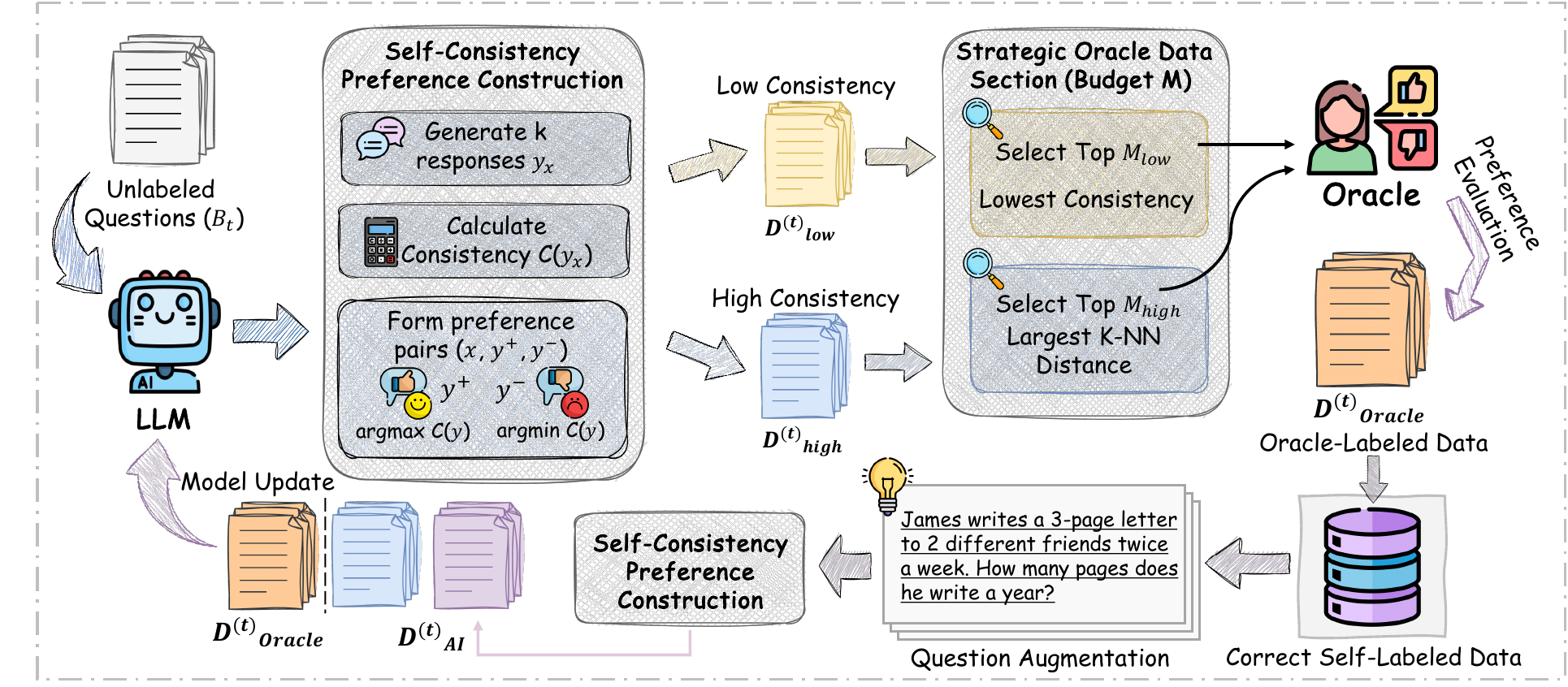}
   \caption{Overview of the \framework\xspace framework. \framework\xspace combines three key components: \ding{192} self-consistency-based preference construction, \ding{193} strategic oracle annotation selection, and \ding{194} oracle-guided instruction augmentation to generate new training data within the model's capability.}
   \label{fig:framework_overview}
   \vspace{-15pt}
\end{figure*}

\paragraph{Self-Consistency Preference Construction.} For each instruction $x$ in the current batch $\mathcal{B}_t$, we use temperature-based sampling with the current model $\theta_t$ to generate $k$ diverse responses: 
\begin{equation}
\abovedisplayskip=1pt
\belowdisplayskip=1pt
y_x = \{y_1, y_2, \ldots, y_k\} \sim \theta_t(\cdot|x) 
\end{equation} where each response $y_i$ includes both the reasoning process and the final answer. This diverse sampling enables us to capture the model's uncertainty~\cite{he2025survey} and reasoning variations for the same instruction. Previous work has demonstrated the promising ability of self-consistency properties of LLMs to improve answer accuracy~\cite{wang2022self,huang-etal-2023-large, yang-etal-2024-weak, prasad2024self,xu-etal-2026-ds2}. The core intuition behind this approach is that while LLMs may produce individual incorrect responses, it is significantly more difficult for them to consistently generate the same erroneous answer across multiple independent sampling attempts. 

Therefore, we employ a consistency function $C(\cdot)$ to measure response consistency and then further construct preference pairs based on self-consistency. The consistency function extracts the final answer from each response $y \in y_x$ via $\text{ans}(\cdot)$ and computes the relative frequency:
\begin{equation}
\abovedisplayskip=1pt
\belowdisplayskip=1pt
C(y) = \frac{1}{k} \sum_{m=1}^k \mathbf{1}\{\text{ans}(y_m) = \text{ans}(y)\}
\end{equation} Using this consistency function, we create initial preference pairs $\mathcal{D}_{self}^{(t)}$ for the current batch $\mathcal{B}_t$ by selecting responses with extreme consistency scores. Specifically, we identify the most consistent response as the chosen response $y^+ = \arg\max_{y \in y_x} C(y)$ and the least consistent response as the rejected response $y^- = \arg\min_{y \in y_x} C(y)$. The preference pairs are then constructed as:
\begin{equation}
\abovedisplayskip=1pt
\belowdisplayskip=1pt
\mathcal{D}_{self}^{(t)} = \{(x, y^+, y^-) \mid x \in \mathcal{B}_t\}
\end{equation}
Based on the consistency score of the chosen response, we further partition the preference pairs into two subsets: 
\begin{equation}
\begin{aligned}
\mathcal{D}_{high}^{(t)} &= \{(x, y^+, y^-) \in \mathcal{D}_{self}^{(t)} : C(y^+) \geq \tau\} \\
\mathcal{D}_{low}^{(t)} &= \{(x, y^+, y^-) \in \mathcal{D}_{self}^{(t)} : C(y^+) < \tau\}
\end{aligned}
\end{equation} where $\tau$ is a consistency threshold that separates high-confidence self-labels from uncertain cases. 

\paragraph{Data Selection for Oracle Annotation.}
In traditional active learning pipelines, the goal is to select the most informative samples for annotation under strict budget constraints~\cite{li2024survey, shenactive, lin2025activedpo}. However, this approach faces a fundamental limitation: the constrained labeling budget prevents the utilization of potentially valuable information contained in the remaining unlabeled data. This raises a critical question: \textit{Given limited annotation budgets, can we leverage both oracle-labeled data and LLM self-labeled data together to synergistically boost model performance?}
\begin{tcolorbox}[colback=white,colframe={rgb,255:red,202;green,200;blue,239},title=Oracle Feedback Protocol]
\begin{itemize}[leftmargin=*,itemsep=0.2pt,topsep=1pt,label={}]
    \item {\color{purple!70!blue}\textbf{Preference Evaluation}}: Verify whether $y^+ \succ y^-$ is correct; flip to $y^- \succ y^+$ if incorrect
\end{itemize}
\end{tcolorbox}
Given an oracle labeling budget of $M$ samples for iteration $t$, we strategically allocate this budget between low-consistency and high-consistency subsets to maximize information gain. Our approach partitions the budget as $M = M_{low} + M_{high}$, where $M_{low}$ and $M_{high}$ represent the number of samples selected from $\mathcal{D}_{low}^{(t)}$ and $\mathcal{D}_{high}^{(t)}$, respectively.

\begin{itemize}[label=$\circ$, leftmargin=*, itemsep=0.1pt, topsep=0.2pt]
    \item \textit{Low-Consistency Selection for Oracle Labeling:} For the low-consistency subset, we select the top $M_{low}$ samples with the lowest consistency scores for their chosen responses, as these represent the most uncertain cases:
\begin{equation}
\mathcal{S}_{low}^{(t)} = \text{TopK}_{\text{lowest}}(\mathcal{D}_{low}^{(t)}, M_{low}, C(y^+))
\end{equation}
where $\text{TopK}_{\text{lowest}}$ selects the $M_{low}$ samples with the smallest consistency scores $C(y^+)$.
    \item \textit{High-Consistency Selection for Oracle Labeling:} While high consistency typically indicates higher reliability, LLMs may still generate self-consistent but incorrect responses. These errors can be especially harmful because they reinforce flawed reasoning patterns~\citep{tan-etal-2025-consistent, duan2024shifting}. To identify such errors in the high-consistency subset $\mathcal{D}_{high}^{(t)}$, we adopt a non-parametric k-nearest neighbors (k-NN) approach inspired by research in out-of-distribution (OOD) detection~\cite{sun2022knn,xu-ding-2025-large,yang-etal-2025-ad,yang2022openood}.
    
   We hypothesize that LLMs generate high-consistency errors on instructions deviating from correctly solved problems. Using oracle-verified correct preferences from previous iterations as in-distribution (ID) data: $\mathcal{D}_{\text{ID}}^{(t)} = \{(x, y^+, y^-) \in \mathcal{S}_{\text{oracle}}^{(t-1)} : y^+ \succ y^- \text{ correct}\}$. For each $x_i \in \mathcal{D}_{\text{high}}^{(t)}$, we extract normalized penultimate hidden states: $z_i = \phi(x_i)/\|\phi(x_i)\|_2$, where $\phi$ is the model's feature encoder. We compute the k-NN distance: $r_k(z_i) = \min_{z \in \mathcal{Z}_{\text{ID}}^{(t)}} \|z_i - z\|_2$ where $\mathcal{Z}_{\text{ID}}^{(t)} = \{\phi(x)/\|\phi(x)\|_2 : (x, y^+, y^-) \in \mathcal{D}_{\text{ID}}^{(t)}\}$. Larger k-NN distances indicate likely OOD samples with self-consistent errors. We select:
\begin{equation}
\abovedisplayskip=2pt
\belowdisplayskip=2pt
\mathcal{S}_{\text{high}}^{(t)} = \text{TopK}_{\text{largest}}(\mathcal{D}_{\text{high}}^{(t)}, M_{\text{high}}, r_k(z_i))
\end{equation}
\end{itemize}
The final oracle evaluation set $\mathcal{D}_{oracle}^{(t)} = \mathcal{S}_{low}^{(t)} \cup \mathcal{S}_{high}^{(t)}$ undergoes human assessment to obtain preference labels.

\vspace{-8pt}
\paragraph{Question Augmentation with Oracle Feedback.} 
Prior work has demonstrated that expanding question diversity to cover a broader range of unseen scenarios effectively improves performance~\citep{yu2023metamath, prasad2024self}. To this end, we exploit oracle annotation for a dual purpose: it provides gold labels for samples where the LLM is most uncertain and reveals cases where the model can provide accurate responses. We leverage these oracle-verified examples as in-context demonstrations to guide the LLM in generating new, diverse instructions that remain within its solvable capability. Specifically, we extract the subset of high-consistency samples with correct preferences:
\begin{equation} 
\abovedisplayskip=3pt
\belowdisplayskip=3pt
\begin{split}
\mathcal{D}_{\text{correct}}^{(t)} = \{(x, y^+, y^-) \in \mathcal{S}_{\text{high}}^{(t)} : \\
y^+ \succ y^- \text{ is correct}\} 
\end{split}
\end{equation}
We randomly sample $n$ instruction examples from $\mathcal{D}_{\text{correct}}^{(t)}$ and prompt the model to generate new instructions:
\begin{equation}
\abovedisplayskip=1pt
\belowdisplayskip=1pt
\mathcal{D}_{\text{new}}^{(t)} = \{x'_i \sim \theta_t(\cdot \mid \text{ICL}(\mathcal{D}_{\text{correct}}^{(t)}, n))\}_{i=1}^{N_{\text{new}}}
\end{equation}
where $\text{ICL}(\mathcal{D}_{\text{correct}}^{(t)}, n)$ constructs an in-context learning prompt from $n$ sampled instructions, generating $N_{\text{new}}$ new instructions $\{x'_i\}_{i=1}^{N_{\text{new}}}$. For each newly generated instruction $x'_i \in \mathcal{D}_{\text{new}}^{(t)}$, we apply the same self-consistency preference construction procedure to generate $k$ responses and construct preference pairs $(x'_i, y'^+_i, y'^-_i)$. We filter these pairs by the consistency threshold and combine them with the original high-consistency self-labeled pairs to form the final AI-labeled dataset:
\begin{equation}
\abovedisplayskip=1pt
\belowdisplayskip=1pt
\begin{aligned}
\mathcal{D}_{\text{AI}}^{(t)} = &(\mathcal{D}_{\text{high}}^{(t)} \setminus \mathcal{S}_{\text{high}}^{(t)}) \cup \{(x'_i, y'^+_i, y'^-_i) \\
&: x'_i \in \mathcal{D}_{\text{new}}^{(t)}, C(y'^+_i) \geq \tau\}
\end{aligned}
\end{equation}

\paragraph{Model Update.} 
We combine oracle-labeled and AI-labeled pairs to construct $\mathcal{D}_{\text{final}}^{(t)} = \mathcal{D}_{\text{oracle}}^{(t)} \cup \mathcal{D}_{\text{AI}}^{(t)}$. We update $\theta_t$ using a modified DPO objective~\citep{pang2024iterative}:
\begin{equation}
\abovedisplayskip=2pt
\belowdisplayskip=2pt
\begin{aligned}
&\mathcal{L}(\theta_t) = -\mathbb{E}_{(x,y^+,y^-) \sim \mathcal{D}_{\text{final}}^{(t)}} \Bigg[ \log \sigma\Big(\beta \log \frac{\theta_t(y^+|x)}{\theta_0(y^+|x)} \\ &- 
\beta \log \frac{\theta_t(y^-|x)}{\theta_0(y^-|x)}\Big) - \alpha |y^+| \log \theta_t(y^+|x) \Bigg]
\end{aligned}
\end{equation} where $\beta > 0$ controls preference learning strength, $\alpha \geq 0$ regulates the likelihood term, and $|y^+|$ weights by response length. The updated model $\theta_{t+1}$ is then used to generate responses for the next iteration $(t+1)$, enabling iterative improvement where each round builds upon the refined capabilities of the previous model. This AI-human supervision strategy is theoretically grounded:
\begin{tcolorbox}[breakable,
    colback=white,
    colframe={rgb,255:red,202;green,200;blue,239},
    title={\textbf{When Does Mixed Supervision Help?}},
    left=2mm, right=2mm, top=1mm, bottom=1mm
]
\textbf{Theorem} (see \S\ref{app:theoretical_analysis} for proof): Let $\mathcal{D}_{\mathrm{oracle}}$ be clean with size $N_o$ and $\mathcal{D}_{\mathrm{AI}}$ have size $N_{ai}$ with symmetric noise $\epsilon_{ai}<\tfrac12$. Let $\mathrm{Gap}(\pi) = V^*(\pi^*) - V^*(\pi)$ denote the policy sub-optimality. Then
\vspace{-5mm}
\[
\frac{\mathrm{Gap}(\pi_{\mathrm{oracle}})}{\mathrm{Gap}(\pi_{\mathrm{mix}})}
\;\ge\; \sqrt{1+\frac{N_{ai}(1-2\epsilon_{ai})^2}{N_o}}.
\]
\vspace{-5mm}
\end{tcolorbox}
\vspace{-5pt}

\section{Experiments}
\label{sec:experiments}
\begin{table*}[!t]
\centering
\caption{
Performance across active learning iterations on reasoning benchmarks with different base models. Results show accuracy (\%). Numbers with arrows indicate improvement (\red{}) or decline (\blue{}) over base model. We highlight the \hlfirst{best} and \hlsecond{second best} results per iteration.
}
\vspace{-0.4em}
\label{tab:active_learning_results}
\renewcommand\tabcolsep{5pt}
\renewcommand\arraystretch{1.08}

\scalebox{0.78}{
\begin{tabular}{c c l|ccccc}
\Xhline{1.2pt}
\rowcolor{blue!10}
{\textbf{Base Model}} & {\textbf{Dataset}} & \textbf{Methods}  & \textbf{Random} & \textbf{Entropy} & \textbf{Pref Certainty} & \textbf{Pref + Ent} & \textbf{\framework} \\
\Xhline{1.2pt}

\multirow{20}{*}{\makecell{Llama3 \\ {\small\colorbox{gray!70}{\textcolor{white}{8B}}}}}
& \multirow{5}{*}{\texttt{GSM8K}}
& Base Model & 23.53 & 23.53 & 23.53 & 23.53 & 23.53 \\
\cmidrule{3-8}
& & \quad \ding{59} \texttt{Iteration 1} & 25.34\red{1.81} & 26.43\red{2.90} & 27.15\red{3.62} & \hlsecond{28.92\red{5.39}} & \hlfirst{31.95\red{8.42}} \\
& & \quad \ding{59} \texttt{Iteration 2} & 28.05\red{4.52} & 30.41\red{6.88} & \hlsecond{32.48\red{8.95}} & 31.76\red{8.23} & \hlfirst{37.56\red{14.03}} \\
& & \quad \ding{59} \texttt{Iteration 3} & 31.31\red{7.78} & 33.21\red{9.68} & 34.62\red{11.09} & \hlsecond{35.89\red{12.36}} & \hlfirst{40.63\red{17.10}} \\
& & \quad \ding{59} \texttt{Iteration 4} & 34.57\red{11.04} & 36.56\red{13.03} & 37.28\red{13.75} & \hlsecond{39.41\red{15.88}} & \hlfirst{43.58\red{20.05}} \\
\cmidrule{2-8}

& \multirow{5}{*}{\texttt{MATH}}
& Base Model & 4.62 & 4.62 & 4.62 & 4.62 & 4.62 \\
\cmidrule{3-8}
& & \quad \ding{59} \texttt{Iteration 1} & \hlfirst{11.44\red{6.82}} & \hlsecond{11.24\red{6.62}} & 11.02\red{6.40} & 11.34\red{6.72} & 8.46\red{3.84} \\
& & \quad \ding{59} \texttt{Iteration 2} & \hlfirst{11.78\red{7.16}} & \hlsecond{11.76\red{7.14}} & 10.84\red{6.22} & 11.45\red{6.83} & 10.88\red{6.26} \\
& & \quad \ding{59} \texttt{Iteration 3} & 11.89\red{7.27} & 12.01\red{7.39} & 11.72\red{7.10} & \hlsecond{11.94\red{7.32}} & \hlfirst{13.21\red{8.59}} \\
& & \quad \ding{59} \texttt{Iteration 4} & 12.07\red{7.45} & \hlsecond{12.94\red{8.32}} & 11.08\red{6.46} & 13.07\red{8.45} & \hlfirst{14.46\red{9.84}} \\
\cmidrule{2-8}

& \multirow{5}{*}{\texttt{WebInstruct}}
& Base Model & 7.69 & 7.69 & 7.69 & 7.69 & 7.69 \\
\cmidrule{3-8}
& & \quad \ding{59} \texttt{Iteration 1} & 7.05\blue{0.64} & 7.51\blue{0.18} & 9.23\red{1.54} & \hlsecond{10.26\red{2.57}} & \hlfirst{11.54\red{3.85}} \\
& & \quad \ding{59} \texttt{Iteration 2} & 4.69\blue{3.00} & 7.93\red{0.24} & \hlsecond{13.46\red{5.77}} & 12.18\red{4.49} & \hlfirst{15.38\red{7.69}} \\
& & \quad \ding{59} \texttt{Iteration 3} & 5.13\blue{2.56} & 8.82\red{1.13} & 11.97\red{4.28} & \hlsecond{13.08\red{5.39}} & \hlfirst{13.82\red{6.13}} \\
& & \quad \ding{59} \texttt{Iteration 4} & 9.62\red{1.93} & 10.11\red{2.42} & 14.53\red{6.84} & \hlsecond{14.87\red{7.18}} & \hlfirst{15.97\red{8.28}} \\

\midrule

\multirow{20}{*}{\makecell{Qwen3 \\ {\small\colorbox{gray!70}{\textcolor{white}{4B}}}}}
& \multirow{5}{*}{\texttt{GSM8K}}
& Base Model & 88.39 & 88.39 & 88.39 & 88.39 & 88.39 \\
\cmidrule{3-8}
& & \quad \ding{59} \texttt{Iteration 1} & 92.35\red{3.96} & 93.51\red{5.12} & \hlfirst{94.12\red{5.73}} & \hlsecond{92.67\red{4.28}} & 93.44\red{5.05} \\
& & \quad \ding{59} \texttt{Iteration 2} & 93.48\red{5.09} & 93.08\red{4.69} & \hlfirst{94.58\red{6.19}} & 93.21\red{4.82} & \hlsecond{94.75\red{6.36}} \\
& & \quad \ding{59} \texttt{Iteration 3} & 94.14\red{5.75} & 93.67\red{5.28} & 94.93\red{6.54} & \hlsecond{94.31\red{5.92}} & \hlfirst{95.02\red{6.63}} \\
& & \quad \ding{59} \texttt{Iteration 4} & 93.57\red{5.18} & 94.02\red{5.63} & 94.03\red{5.64} & \hlsecond{94.58\red{6.19}} & \hlfirst{94.84\red{6.45}} \\
\cmidrule{2-8}

& \multirow{5}{*}{\texttt{MATH}}
& Base Model & 69.17 & 69.17 & 69.17 & 69.17 & 69.17 \\
\cmidrule{3-8}
& & \quad \ding{59} \texttt{Iteration 1} & 68.78\blue{0.39} & 68.24\blue{0.93} & 68.46\blue{0.71} & 69.12\blue{0.05} & \hlfirst{73.91\red{4.74}} \\
& & \quad \ding{59} \texttt{Iteration 2} & 69.28\red{0.11} & 68.94\blue{0.23} & 69.24\red{0.07} & 69.87\red{0.70} & \hlfirst{74.39\red{5.22}} \\
& & \quad \ding{59} \texttt{Iteration 3} & 69.43\red{0.26} & 69.38\red{0.21} & 69.01\blue{0.16} & 71.68\red{2.51} & \hlfirst{75.64\red{6.47}} \\
& & \quad \ding{59} \texttt{Iteration 4} & 70.89\red{1.72} & 70.14\red{0.97} & 70.52\red{1.35} & 70.21\red{1.04} & \hlfirst{75.71\red{6.54}} \\
\cmidrule{2-8}

& \multirow{5}{*}{\texttt{WebInstruct}}
& Base Model & 35.92 & 35.92 & 35.92 & 35.92 & 35.92 \\
\cmidrule{3-8}
& & \quad \ding{59} \texttt{Iteration 1} & 37.68\red{1.76} & 39.14\red{3.22} & 41.25\red{5.33} & \hlsecond{42.87\red{6.95}} & \hlfirst{44.23\red{8.31}} \\
& & \quad \ding{59} \texttt{Iteration 2} & 40.15\red{4.23} & 43.26\red{7.34} & \hlsecond{47.92\red{12.00}} & 46.73\red{10.81} & \hlfirst{50.38\red{14.46}} \\
& & \quad \ding{59} \texttt{Iteration 3} & 42.37\red{6.45} & 45.89\red{9.97} & 49.64\red{13.72} & \hlsecond{50.21\red{14.29}} & \hlfirst{52.77\red{16.85}} \\
& & \quad \ding{59} \texttt{Iteration 4} & 44.12\red{8.20} & 47.83\red{11.91} & 51.25\red{15.33} & \hlsecond{52.48\red{16.56}} & \hlfirst{53.96\red{18.04}} \\

\Xhline{1.2pt} 
\end{tabular}
}
\vspace{-1.4em}
\end{table*}
\subsection{Experimental Setup}
\begin{figure*}[t]
   \centering\
   \includegraphics[width=0.98\textwidth]{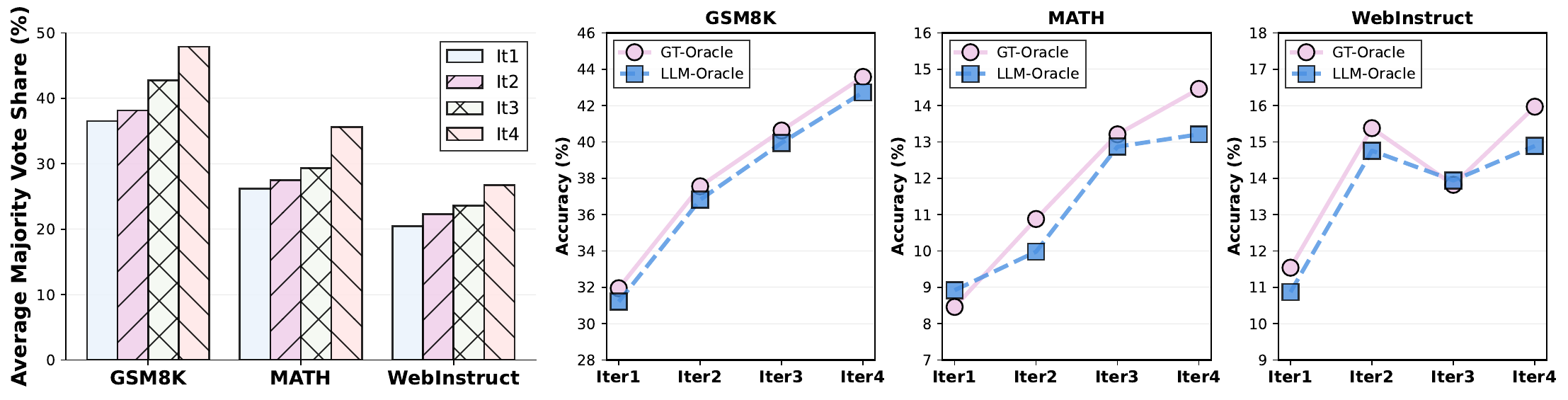}
   \caption{Left: Average majority vote share across iterations, showing the percentage of samples where the most frequent answer from $k$ sampled responses matches the ground truth. Right: Comparison of GT-Oracle vs LLM-Oracle across iterations.}
   \label{fig:combined_analysis}
   \vspace{-15pt}
\end{figure*}

\noindent\textbf{Datasets.} To evaluate the effectiveness of our proposed \framework, we utilize three commonly used reasoning benchmarks for our main experiment: \texttt{GSM8K}~\cite{cobbe2021training} for grade school mathematical reasoning, \texttt{MATH}~\cite{hendrycks2021measuring} for advanced competition-level mathematics, and \texttt{WebInstruct}~\cite{ma2025general} for physics reasoning. To assess generalization to out-of-domain data, we further evaluate on \texttt{AIME}~\citep{aime_1983_2024} for advanced mathematical problem-solving, \texttt{GPQA}~\citep{rein2024gpqa} for graduate-level science questions, and \texttt{MMLU-Pro}~\citep{wang2024mmlu} for multi-domain knowledge. Detailed dataset statistics are provided in Appendix~\ref{app:datasets}.

\noindent\textbf{Baselines.} We compare our approach against several active learning methods for preference alignment, focusing on different data selection strategies for oracle annotation while maintaining consistent training procedures across all methods: (1) \textit{Random} randomly selects preference pairs without informativeness criteria; (2) \textit{Entropy}~\cite{muldrew2024active} selects samples with highest prediction entropy to target model uncertainty; (3) \textit{Pref Certainty}~\cite{muldrew2024active} prioritizes samples with low confidence in preference predictions; and (4) \textit{Pref + Ent}~\cite{muldrew2024active} combines preference uncertainty with prediction entropy. All baselines follow identical training protocols using the same DPO objective and hyperparameters, differing only in their oracle annotation selection strategies. Detailed descriptions of each baseline method are summarized in Appendix~\ref{app:baselines}.

\noindent\textbf{Implementation Details.} We use \texttt{Llama3-8B}~\citep{grattafiori2024llama} and \texttt{Qwen3-4B}~\citep{yang2025qwen3} as backbone models to demonstrate effectiveness across different model families and scales. For each iteration, we set the oracle budget to $M = 300$ and train all models using the modified DPO objective for 10 epochs with learning rate $5 \times 10^{-6}$ and effective batch size 16. We set the DPO hyperparameter $\beta = 0.5$, NLL regularization coefficient $\alpha = 1$, and the number of sampled responses $k = 8$. When generating multiple responses and questions, we sample using temperatures from the set $\{0.35, 0.4, 0.45, 0.5, 0.55, 0.6, 0.65, 0.7\}$ to encourage diverse reasoning paths.
Additional details are in Appendix~\ref{app:implementations}.
\subsection{Main Results}
Table~\ref{tab:active_learning_results} presents the performance of \framework~compared to baseline active learning methods across three reasoning benchmarks and two model families. We make three key observations: \textbf{Observation \ding{192}: \framework~consistently outperforms baselines in later iterations.} Across both model families and all three datasets, \framework~achieves the best performance by iteration 4. On \texttt{Llama3-8B}, \framework~improves over the base model by 20.05\% on \texttt{GSM8K}, 9.84\% on \texttt{MATH}, and 8.28\% on \texttt{WebInstruct}. However, we observe that on \texttt{MATH} with \texttt{Llama3-8B}, \framework~underperforms baselines in early iterations. This is likely due to the base model's low initial capability, causing generated responses to be too noisy for reliable self-labeling. \textbf{Observation \ding{193}: Performance variance decreases with stronger base models.} When the base model has strong initial capability, performance differences across methods become less pronounced. On \texttt{Qwen3-4B} with \texttt{GSM8K}, the performance spread at iteration 4 is only 1.27\%. In contrast, on \texttt{Llama3-8B} with \texttt{GSM8K}, the spread is 9.01\%. This suggests that strategic data selection matters more when the base model has substantial room for improvement, while stronger models benefit more uniformly from additional preference data regardless of selection strategy. \textbf{Observation \ding{194}: Strategic mixing of oracle and self-labeled data is crucial.} The variance in baseline performance highlights the importance of data utilization. Random sampling often shows the weakest performance, particularly on \texttt{WebInstruct} with \texttt{Llama3-8B} where it degrades performance.

\subsection{More Analysis}
To provide deeper insights into \framework, we analyze \ding{172} self-consistency evolution across iterations, \ding{173} consistency-accuracy correlation, \ding{174} oracle design choices, and \ding{175} out-of-domain generalization, all using \texttt{Llama3-8B}.
\begin{figure*}[t]
\centering
\begin{tabular}{@{}c@{\hspace{0.02\textwidth}}c@{}}
\begin{minipage}{0.65\textwidth}
   \centering
   \includegraphics[width=\textwidth]{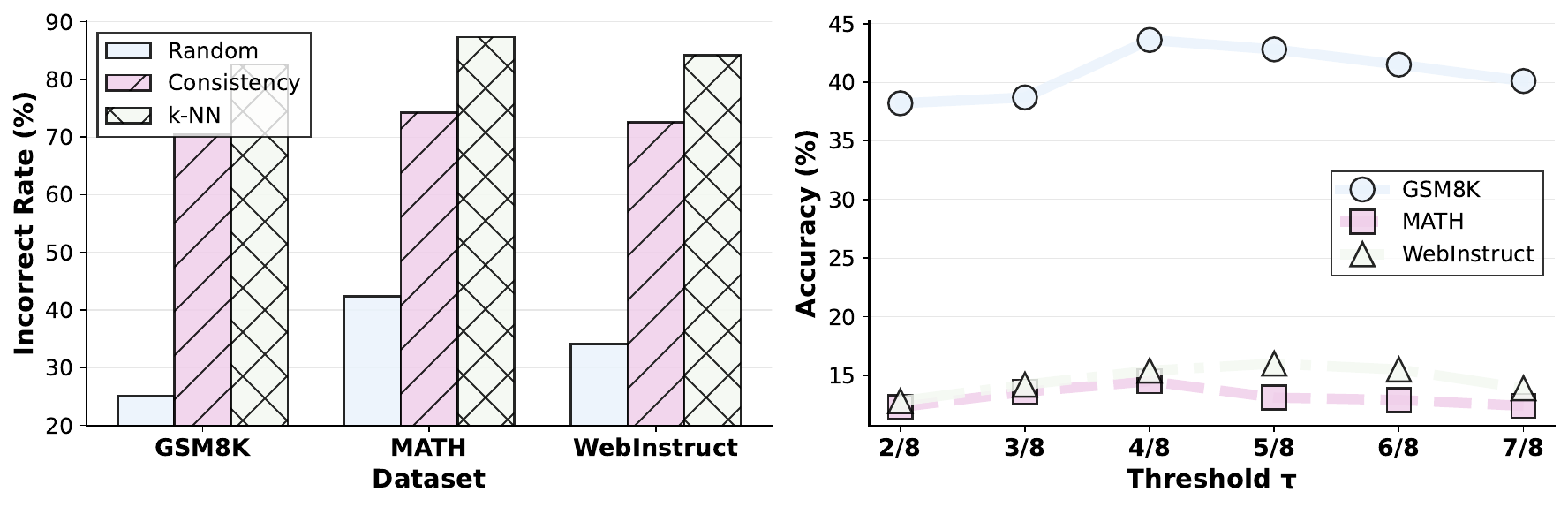}
\end{minipage}
&
\begin{minipage}{0.32\textwidth}
   \centering
   \renewcommand\tabcolsep{3pt}
   \renewcommand\arraystretch{1.2}
   \footnotesize
   \begin{tabular}{cc|ccc}
   \Xhline{1.2pt}
   \rowcolor{blue!10}
   \textbf{Aug.} & \textbf{Self} & \textbf{GSM8K} & \textbf{MATH} & \textbf{Web.} \\
   \Xhline{1pt}
   \rowcolor{white}
   $\circ$ & $\circ$ &  36.38 & 11.95 & 13.53 \\
   \rowcolor{white}
   \ding{52} & $\circ$ & 39.85 & 13.34 & 13.89 \\
   \rowcolor{white}
   $\circ$ & \ding{52} & 41.23 & 12.12 & 15.21 \\
   \rowcolor{white}
   \ding{52} & \ding{52} & \textbf{43.58} & \textbf{14.46} & \textbf{15.97} \\
   \Xhline{1.2pt}
   \end{tabular}
\end{minipage}
\end{tabular}
\caption{Sensitivity analysis and ablation study. Left: Incorrect rate for different high-consistency selection methods. Middle: Performance across consistency threshold $\tau$ values. Right: Ablation study of key modules.}
\label{fig:sensitivity_ablation}
\vspace{-15pt}
\end{figure*}

\noindent \textbf{Self-Consistency Change Over Iterations.}
From Figure~\ref{fig:combined_analysis} (left), we observe that models become more consistent across iterations, with the average vote share $\mathcal{C}(y^+)$ increasing steadily on all datasets. Notably, we find that datasets where the model achieves higher performance also exhibit stronger self-consistency. This positive correlation between consistency and accuracy validates our approach of using self-consistency as a reliable proxy for response quality in constructing preference pairs.
\begin{table}[htbp]
\centering
\caption{Pearson correlation between self-consistency $\mathcal{C }(y)$ and accuracy on test set across iterations.}
\vspace{-5pt}
\label{tab:consistency_accuracy_correlation}
\footnotesize
\renewcommand\tabcolsep{3pt}
\renewcommand\arraystretch{1.2}
\resizebox{\columnwidth}{!}{%
\begin{tabular}{l|cccc}
\Xhline{1.2pt}
\rowcolor{blue!10}
\textbf{Dataset} & \textbf{Iteration 1} & \textbf{Iteration 2} & \textbf{Iteration 3} & \textbf{Iteration 4} \\
\Xhline{1pt}
\texttt{GSM8K} & 0.8654 & 0.8721 & 0.9582 & 0.9745 \\
\texttt{MATH} & 0.9135 & 0.9429 & 0.9601 & 0.9756 \\
\texttt{WebInstruct} & 0.7815 & 0.8623 & 0.9183 & 0.9544 \\
\Xhline{1.2pt}
\end{tabular}%
}
\vspace{-15pt}
\end{table}

\noindent\textbf{Self-Consistency vs. Accuracy.} After each iteration, we evaluate the trained model on the test set by generating $k$ responses per instruction and constructing preference pairs using our self-consistency criterion. For each test instruction $x$, we compute the consistency score $C(y^+)$ of the chosen response and compare it against the ground-truth correctness. Table~\ref{tab:consistency_accuracy_correlation} reports the Pearson correlation between self-consistency scores and test accuracy across training iterations on \texttt{GSM8K}, \texttt{MATH}, and \texttt{WebInstruct}. Across all datasets, self-consistency exhibits a strong positive correlation with accuracy, and this correlation consistently strengthens over successive iterations.

\noindent\textbf{Choice of Oracle.}
We compare two oracle scenarios to validate our design choices. In our main experiments, we employ GPT-5 as the oracle LLM and also provide ground-truth answers as references to judge preference pairs. We compare this approach with an alternative where GPT-5 serves as the oracle based purely on its own judgment. Figure~\ref{fig:combined_analysis} presents the performance comparison across \texttt{GSM8K}, \texttt{MATH}, and \texttt{WebInstruct} datasets over four training iterations. Across all three benchmarks, both oracle configurations yield comparative performance. These results indicate that powerful LLMs can serve as effective substitutes for humans, substantially reducing annotation costs while maintaining strong performance.

\noindent\textbf{Out-of-Domain Generalization.}
We evaluate out-of-domain generalization on \texttt{GPQA} and \texttt{MMLU-Pro} using models trained on \texttt{GSM8K}, \texttt{MATH}, and \texttt{WebInstruct}. 
As shown in Figure~\ref{fig:ood_generalization}, \framework~consistently outperforms all baselines across training datasets, demonstrating robust transfer to unseen domains. We hypothesize that oracle-guided instruction augmentation contributes to this improved generalization by diversifying the training distribution within the model's solvable range. A more comprehensive analysis is provided in Appendix~\ref{app:experiments}.

\begin{figure}[h]
    \centering
    \includegraphics[width=0.9\linewidth]{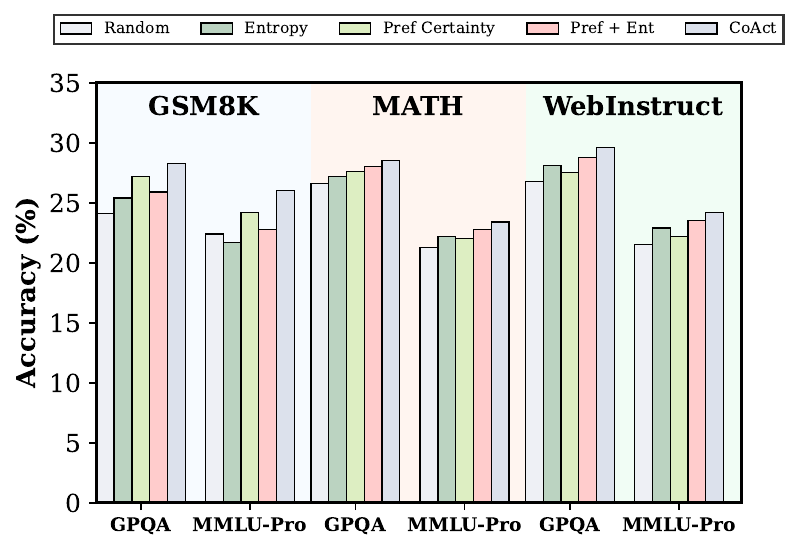}
    \vspace{-5pt}
    \caption{Out-of-domain generalization on \texttt{GPQA} and \texttt{MMLU-Pro}.}
    \label{fig:ood_generalization}
    \vspace{-15pt}
\end{figure}

\subsection{Ablation and Sensitivity Analysis}
In this section, we conduct sensitivity analysis of key hyperparameters and evaluate the effectiveness of the key modules in \framework.

\noindent\textbf{Effectiveness of k-NN Selection.} To validate the effectiveness of k-NN distance for detecting self-consistent errors in high-consistency samples, we compare different selection strategies. Figure~\ref{fig:sensitivity_ablation} (left) shows the oracle incorrect rate for samples selected by each method at Iteration 4 on \texttt{Llama3-8B}. The k-NN distance approach consistently identifies samples with significantly higher error rates across all datasets, achieving 82.56\% on \texttt{GSM8K}, 87.32\% on \texttt{MATH}, and 84.17\% on \texttt{WebInstruct}. In contrast, random selection yields substantially lower error rates, while selecting the lowest consistency samples in this subset results in 70.42\%, 74.23\%, and 72.56\%. This demonstrates that k-NN distance successfully identifies OOD samples where the model generates self-consistent but incorrect responses.

\noindent\textbf{Impact of Consistency Threshold.} We analyze the sensitivity of \framework~to the consistency threshold $\tau$ that partitions samples into high and low-consistency subsets. Figure~\ref{fig:sensitivity_ablation} (middle) shows performance across different threshold values at Iteration 4 on \texttt{Llama3-8B}. Performance peaks around $\tau = 4/8$ to $5/8$ across datasets, with degradation at both lower and higher values. When $\tau$ is too low, unreliable samples contaminate the self-labeled training data; when too high, the effective training set size remains low. 

\noindent\textbf{Ablation Study.} To understand the contribution of each component in \framework, we conduct an ablation study by removing key components and evaluating performance at the final iteration on \texttt{Llama3-8B}. Table~\ref{tab:consistency_accuracy_correlation} presents the results.Without both question augmentation (Aug.) and self-labeling, the framework relies solely on oracle-labeled data from active selection, achieving 36.38\% on \texttt{GSM8K}, 11.95\% on \texttt{MATH}, and 13.53\% on \texttt{WebInstruct}. This oracle-only baseline demonstrates that limited annotation budget alone is insufficient to fully unlock model capabilities. Interestingly, on \texttt{MATH}, using question augmentation outperforms using self-labeling alone, suggesting that oracle-guided instruction generation within the model's solvable range is more effective than leveraging self-consistency on the original unlabeled instructions if the original data is too hard. 
\section{Conclusion}
\label{sec:conclusion}
In this paper, we introduce \framework, a framework that bridges self-rewarding and active preference learning through strategic human-AI collaboration. Our approach addresses the fundamental dilemma in preference alignment: balancing the scalability of AI-generated preference data with the quality assurance of human annotations. Experimental results demonstrate that \framework~achieves demonstrative performance compared to existing active preference learning pipelines. These findings establish \framework~as a practical solution for preference alignment in resource-constrained settings.

\section*{Limitations}
While \framework~demonstrates strong empirical performance, we acknowledge several limitations. First, generating multiple responses per instruction for self-consistency increases computational overhead. Second, our evaluation focuses on reasoning tasks with objective answers. Third, the consistency threshold requires dataset-specific tuning, though performance remains relatively stable across reasonable ranges. These limitations suggest promising directions for extending \framework~to broader domains and more efficient consistency estimation methods.

\bibliography{anthology,custom}
\appendix
\section{Datasets}
\label{app:datasets}

We evaluate \framework~on three reasoning benchmarks covering different domains and difficulty levels:
\begin{itemize}[label=$\circ$, leftmargin=*, itemsep=0.1pt, topsep=0.2pt]
\item\textbf{GSM8K~\citep{cobbe2021training}.} GSM8K (Grade School Math 8K) is a dataset of 8,500 grade school math word problems that require multi-step arithmetic reasoning. Each problem is accompanied by a natural language solution with intermediate reasoning steps and a final numerical answer. Following previous work~\citep{prasad2024self}, we use a data split of 6.7K/0.8K/1.3K for train/dev/test sets respectively. We use the training set as our unlabeled instruction pool, the dev set for validation, and report final performance on the test set. For each iteration, we sample a batch of 1,675 instructions from the unlabeled pool. Performance is measured using exact match accuracy, where a response is considered correct only if the final numerical answer exactly matches the ground truth.
\item \textbf{MATH~\citep{hendrycks2021measuring}.} The MATH dataset contains 12,500 challenging competition-level mathematics problems from high school math competitions. Problems span various topics including algebra, counting and probability, geometry, intermediate algebra, number theory, prealgebra, and precalculus. Each problem includes a detailed solution with step-by-step reasoning. We hold out a portion of the training set to create a dev set for model selection and hyperparameter tuning, resulting in train/dev/test splits of 6.7K/0.8K/5K problems respectively. We use the training set as our unlabeled instruction pool and report exact match accuracy on the test set. For each iteration, we sample a batch of 1,675 instructions from the unlabeled pool.
\item \textbf{WebInstruct~\citep{ma2025general}.} WebInstruct is a reasoning dataset designed to evaluate models' ability to solve complex science problems across diverse domains. In this paper, we focus on the physics domain. We use the WebInstruct-verified dataset from TIGER-Lab and filter samples by answer type \texttt{float} to obtain numerical physics problems. This results in train/dev/test splits of 8K/1K/156 samples respectively. The training set is further divided into 4 iterations of 2K samples each. For each iteration, we sample a batch of 500 instructions from the current iteration's pool. We report exact match accuracy on the test set.
\end{itemize}

\paragraph{Dataset Statistics.} Table~\ref{tab:dataset_stats} summarizes key statistics for the three reasoning benchmarks used in our experiments. Table~\ref{tab:iteration_stats} reports the total training data size per iteration for both \texttt{Llama3-8B} and \texttt{Qwen3-4B}, which includes oracle-labeled samples, self-labeled samples, and augmented questions. Table~\ref{tab:gpt5_agreement} presents the agreement rate between GPT-5 judgments and the original preference labels using self-consistency construction. We observe consistently high agreement rates across all settings, with \texttt{Qwen3-4B} achieving higher agreement due to its stronger initial capabilities.

\begin{table}[h]
\centering
\caption{Dataset statistics for the three reasoning benchmarks used in our experiments.}
\label{tab:dataset_stats}
\small
\renewcommand\tabcolsep{5pt}
\renewcommand\arraystretch{1.2}
\begin{tabular}{l|cccc}
\Xhline{1.2pt}
\rowcolor{blue!10}
\textbf{Dataset} & \textbf{Train} & \textbf{Dev} & \textbf{Test} & \textbf{Domain} \\
\Xhline{1pt}
GSM8K & 6,700 & 800 & 1,319 & Grade school math \\
MATH & 6,700 & 800 & 5,000 & Competition math \\
WebInstruct & 8,000 & 1,000 & 156 & Physics reasoning \\
\Xhline{1.2pt}
\end{tabular}
\end{table}

\begin{table}[h]
\centering
\caption{Total training data size per iteration for \texttt{Llama3-8B} and \texttt{Qwen3-4B}.}
\label{tab:iteration_stats}
\small
\renewcommand\tabcolsep{4pt}
\renewcommand\arraystretch{1.2}
\begin{tabular}{l|c|cccc}
\Xhline{1.2pt}
\rowcolor{blue!10}
\textbf{Dataset} & \textbf{Model} & \textbf{It 1} & \textbf{It 2} & \textbf{It 3} & \textbf{It 4} \\
\Xhline{1pt}
\multirow{2}{*}{GSM8K} 
& \texttt{Llama3-8B} & 865 & 948 & 1065 & 1034 \\
& \texttt{Qwen3-4B} & 1792 & 1356 & 1495 & 1614 \\
\hline
\multirow{2}{*}{MATH} 
& \texttt{Llama3-8B} & 869 & 1124 & 915 & 971 \\
& \texttt{Qwen3-4B} & 1206 & 1032 & 1281 & 1130 \\
\hline
\multirow{2}{*}{WebInstruct} 
& \texttt{Llama3-8B} & 1421 & 1350 & 1382 & 1415 \\
& \texttt{Qwen3-4B} & 1655 & 1558 & 1617 & 1672 \\
\Xhline{1.2pt}
\end{tabular}
\end{table}
\begin{table}[h]
\centering
\caption{Agreement (\%) between GPT-5 judgments and preference labels in the constructed dataset across models, datasets, and iterations.}
\label{tab:gpt5_agreement}
\small
\renewcommand\tabcolsep{4.5pt}
\renewcommand\arraystretch{1.2}
\begin{tabular}{l|c|cccc}
\Xhline{1.2pt}
\rowcolor{blue!10}
\textbf{Dataset} & \textbf{Model} & \textbf{It 1} & \textbf{It 2} & \textbf{It 3} & \textbf{It 4} \\
\Xhline{1pt}
\multirow{2}{*}{GSM8K} 
& Llama3-8B & 79.67 & 81.33 & 82.56 & 81.12 \\
& Qwen3-4B & 91.14 & 92.34 & 93.00 & 92.68 \\
\hline
\multirow{2}{*}{MATH} 
& Llama3-8B & 77.78 & 79.12 & 79.16 & 82.37 \\
& Qwen3-4B & 86.02 & 89.45 & 90.13 & 89.54 \\
\hline
\multirow{2}{*}{WebInstruct} 
& Llama3-8B & 86.43 & 87.24 & 87.78 & 90.18 \\
& Qwen3-4B & 93.36 & 94.45 & 94.35 & 94.22 \\
\Xhline{1.2pt}
\end{tabular}
\end{table}

\section{Baselines}
\label{app:baselines}

We compare \framework~against four active learning baselines for preference alignment. All methods share the same iterative training framework and DPO objective, differing only in their data selection strategies for oracle annotation. At each iteration $t$, each method samples a batch $\mathcal{B}_t$ from the unlabeled pool, selects $M$ samples from $\mathcal{B}_t$ for oracle labeling, and trains the model using the modified DPO objective.

\begin{itemize}[label=$\circ$, leftmargin=*, itemsep=0.1pt, topsep=0.2pt]
\item \textbf{Random.} This baseline randomly selects $M$ samples from the unlabeled instruction pool.

\item \textbf{Entropy~\citep{muldrew2024active}.} This method leverages the predictive entropy of the language model as an uncertainty measure. For each instruction $x$ in the batch $\mathcal{B}_t$, we sample $k$ responses and approximate the predictive entropy as:
\begin{equation}
H_{p_\theta}(y|x) \approx -\frac{1}{k}\sum_{i=1}^k \log p_\theta(y_i|x), \quad y_i \sim p_\theta(\cdot|x)
\end{equation}
The method selects the top $M$ samples with the highest entropy, prioritizing instructions where the model exhibits the greatest uncertainty.

\item \textbf{Preference Certainty (Pref Certainty)~\citep{muldrew2024active}.} This method focuses on the model's confidence in preference predictions under the Bradley-Terry model. For each instruction $x \in \mathcal{B}_t$, we sample two responses $y_1, y_2 \sim p_\theta(\cdot|x)$ and compute the difference in implicit rewards:
\begin{equation}
\begin{split}
\text{certainty}(x) = |\hat{r}(x, y_1) - \hat{r}(x, y_2)|, \\\quad \text{where } \hat{r}(x, y) = \beta \log \frac{p_\theta(y|x)}{p_{\theta_0}(y|x)}    
\end{split}
\end{equation}
The method selects the top $M$ samples with the lowest certainty scores, targeting cases where the model is most uncertain about preference ordering.
\item \textbf{Preference + Entropy (Pref + Ent)~\citep{muldrew2024active}.} This hybrid approach combines entropy and preference certainty to exploit their complementary strengths. The selection operates in two stages:
\begin{enumerate}[leftmargin=*, itemsep=2pt]
    \item \textbf{Entropy Filtering:} Rank all instructions in $\mathcal{B}_t$ by predictive entropy and select the top $K$ samples with highest entropy, where $K > M$.
    \item \textbf{Preference Selection:} For the filtered $K$ samples, generate response pairs and compute preference certainty scores. Select the top $M$ samples with lowest certainty.
\end{enumerate}
This two-stage design is based on the hypothesis that high-entropy instructions are more likely to yield uncertain preference predictions. We set $K = 2M$ following \citet{muldrew2024active}.

\end{itemize}

\section{Implementation Details}
\label{app:implementations}

\subsection{Environment and Models}

All experiments are conducted on 4 NVIDIA A100 GPUs with 80GB memory. We use \texttt{Llama3-8B-Base} and \texttt{Qwen3-4B} as backbone models to demonstrate effectiveness across different model families and scales. Unless otherwise specified, all experiments share the same hardware configuration and training pipeline.

\subsection{Training Setup}

Models are optimized using AdamW with a learning rate of $5\times10^{-6}$ and weight decay 0.01. The effective batch size is 16. Training is performed for up to 10 epochs per iteration, and checkpoints are selected based on accuracy on the development set at the end of each epoch. All models are trained using the modified DPO objective, with DPO temperature $\beta=0.5$ and NLL regularization coefficient $\alpha=1.0$.

For active learning, we allocate an oracle budget of $M=300$ preference pairs per iteration, evenly split between low-consistency and high-consistency subsets, with $M_{\text{low}}=150$ and $M_{\text{high}}=150$. For each prompt, we sample $k=8$ candidate responses. The dataset-specific consistency threshold $\tau$ is set to $4/8$ for GSM8K and MATH, and $5/8$ for WebInstruct. We adopt LoRA with LoRA rank set to $r=8$ and LoRA alpha set tp $\alpha_{\text{LoRA}}=16$. Model training is implemented using the LLaMA-Factory framework.\footnote{\url{https://github.com/hiyouga/LlamaFactory}}
\subsection{Inference and Evaluation}
During response generation, we use nucleus sampling with top-$p=0.9$ and sample temperatures from the set $\{0.35, 0.4, 0.45, 0.5, 0.55, 0.6, 0.65, 0.7\}$ to encourage diverse reasoning paths. For final evaluation, we generate responses using three different random seeds and report results averaged across these runs. In this setting, we fix the generation temperature to 0.7 and top-$p$ to 0.9. Inference is accelerated using \texttt{vLLM}.\footnote{\url{https://github.com/vllm-project/vllm}} All baseline methods use identical inference settings for fair comparison.

\section{Theoretical Analysis}
\label{app:theoretical_analysis}

We formalize preference-based LLM alignment under noisy feedback and establish conditions
under which mixed supervision, combining clean oracle preferences with noisy AI-generated
preferences, provably improves over oracle-only training.

\subsection{Preliminaries and Assumptions}

\begin{assumption}[Preference Generation and BTL Model]
\label{assump:btl}
Prompts $s\sim\rho$ and candidate responses $(a,a')\sim\pi_{\mathrm{ref}}(\cdot\mid s)$.
Preferences are generated under the Bradley--Terry--Luce model~\citep{bradley1952rank}:
\begin{align*}
\mathbb{P}[a\succ a'\mid s] = \sigma(r^*(s,a)-r^*(s,a')),
\end{align*}
where $\sigma(z) = (1+e^{-z})^{-1}$ and $r^*$ is the latent reward.
\end{assumption}

\begin{assumption}[KL-Regularized Optimal Policy]
\label{assump:kl_opt}
The optimal policy $\pi^*$ solves
\begin{align*}
\max_{\pi}\;\mathbb{E}_{s\sim\rho,\,a\sim\pi}\!\left[r^*(s,a)
- \beta\log\tfrac{\pi(a\mid s)}{\pi_{\mathrm{ref}}(a\mid s)}\right],
\end{align*}
with $\beta>0$.
\end{assumption}

Following~\citet{rafailov2024direct}, DPO is equivalent to binary logistic regression
with implicit reward difference $h_\theta(x)$ and clean label probability
$\mu_\theta(x)=\sigma(h_\theta(x))$, where $x=(s,a_w,a_l)$ and $y=1$ indicates
$a_w\succ a_l$.

\begin{assumption}[Regularity]
\label{assump:regularity}
$h_\theta$ is twice continuously differentiable in $\theta$, the clean Fisher information
\begin{align*}
I_{\mathrm{clean}}(\theta)
= \mathbb{E}_x\!\left[\mu_\theta(1-\mu_\theta)\nabla h_\theta \nabla h_\theta^\top\right]
\end{align*}
is positive definite at $\theta^*$ with effective dimension $d$, and $r^*$ is
$L_r$-Lipschitz in $\theta$ near $\theta^*$.
\end{assumption}

\subsection{Fisher Information Under Symmetric Noise}

Let $\tilde y$ be obtained by flipping $y$ with probability $\epsilon\in[0,\tfrac12)$.
Then $\tilde\mu_\theta(x) = (1-2\epsilon)\mu_\theta(x) + \epsilon$.

\begin{lemma}[Noise-Attenuated Fisher Information]
\label{lem:noisy_fisher}
Under Assumption~\ref{assump:regularity} and $\epsilon\in[0,\tfrac12)$,
\begin{align*}
I_{\mathrm{noisy}}(\theta) \;\preceq\; (1-2\epsilon)^2\, I_{\mathrm{clean}}(\theta),
\end{align*}
where $\preceq$ denotes the Loewner order.
\end{lemma}

\begin{proof}
Since $\nabla\tilde\mu_\theta = (1-2\epsilon)\,\mu_\theta(1-\mu_\theta)\,\nabla h_\theta$,
the score of the noisy likelihood is
\begin{align*}
\nabla_\theta\log\tilde p_\theta(\tilde y\mid x)
= \tfrac{(1-2\epsilon)\,\mu_\theta(1-\mu_\theta)}{\tilde\mu_\theta(1-\tilde\mu_\theta)}\,
(\tilde y-\tilde\mu_\theta)\,\nabla h_\theta.
\end{align*}
Taking the outer-product expectation with
$\mathbb{E}[(\tilde y-\tilde\mu_\theta)^2\mid x] = \tilde\mu_\theta(1-\tilde\mu_\theta)$,
\begin{align*}
I_{\mathrm{noisy}}(\theta)
= (1-2\epsilon)^2\,\mathbb{E}_x\!\left[
\tfrac{(\mu_\theta(1-\mu_\theta))^2}{\tilde\mu_\theta(1-\tilde\mu_\theta)}
\nabla h_\theta\nabla h_\theta^\top\right].
\end{align*}
Since $\tilde\mu_\theta$ is a convex combination of $\mu_\theta$ and $\tfrac12$, concavity
of $f(t)=t(1-t)$ gives $\tilde\mu_\theta(1-\tilde\mu_\theta)\ge\mu_\theta(1-\mu_\theta)$.
Substituting yields the claim.
\end{proof}

Lemma~\ref{lem:noisy_fisher} recovers the classical $(1-2\epsilon)$ signal attenuation
for learning with noisy labels~\citep{natarajan2013learning} in the DPO setting.

\subsection{Mixed Supervision: Error Bound via Fisher Additivity}

Let $\mathcal{D}_{\mathrm{oracle}}$ be clean with $N_o$ samples and
$\mathcal{D}_{\mathrm{AI}}$ be self-labeled with $N_{ai}$ samples and symmetric noise
rate $\epsilon_{ai}\in[0,\tfrac12)$, both drawn i.i.d.\ from $\rho$. By Fisher additivity,
\begin{align*}
I_{\mathrm{mix}}(\theta)
&= N_o\bar I_{\mathrm{clean}} + N_{ai}\bar I_{\mathrm{noisy}} \\
&\preceq N_{\mathrm{eff}}\,\bar I_{\mathrm{clean}}(\theta),
\end{align*}
where $N_{\mathrm{eff}} \triangleq N_o + N_{ai}(1-2\epsilon_{ai})^2$ denotes the
\emph{effective sample size}.

\begin{lemma}[Parameter Estimation Rate]
\label{lem:rate}
Under Assumption~\ref{assump:regularity}, as $N_o+N_{ai}\to\infty$,
\begin{align*}
\mathbb{E}\|\hat\theta-\theta^*\|_2
\le (1+o(1))\sqrt{\tfrac{d}{N_{\mathrm{eff}}\,\lambda_{\min}(\bar I_{\mathrm{clean}})}}.
\end{align*}
\end{lemma}

\begin{proof}
By standard M-estimator asymptotics, the MLE satisfies
$\sqrt{N_o+N_{ai}}(\hat\theta-\theta^*)\xrightarrow{d}\mathcal{N}(0,\Sigma)$ with
$\Sigma^{-1}=\tfrac{N_{\mathrm{eff}}}{N_o+N_{ai}}\bar I_{\mathrm{clean}}(\theta^*)$.
Thus $\mathbb{E}\|\hat\theta-\theta^*\|_2^2 = \tfrac{1}{N_{\mathrm{eff}}}\mathrm{tr}(\bar I_{\mathrm{clean}}^{-1}) + o(N_{\mathrm{eff}}^{-1})$.
Using $\mathrm{tr}(\bar I_{\mathrm{clean}}^{-1})\le d/\lambda_{\min}(\bar I_{\mathrm{clean}})$
and Jensen's inequality yields the bound.
\end{proof}

\begin{lemma}[From Parameter Error to Policy Gap]
\label{lem:policy_gap}
Under Assumption~\ref{assump:regularity}, the policy gap
$\mathrm{Gap}(\theta) = V^*(\pi^*)-V^*(\pi_\theta)$ satisfies
\begin{align*}
\mathrm{Gap}(\theta) \le 2L_r\,\|\theta-\theta^*\|_2.
\end{align*}
\end{lemma}

\begin{proof}
Let $V^*_{r_\theta}$ denote the regularized value under $r_\theta$. Adding and subtracting
$V^*_{r_\theta}(\pi^*)$ and $V^*_{r_\theta}(\pi_\theta)$ gives
\begin{align*}
\mathrm{Gap}(\theta)
&= \mathbb{E}_{\pi^*}[r^*-r_\theta] + \mathbb{E}_{\pi_\theta}[r_\theta-r^*] \\
&\quad + \big[V^*_{r_\theta}(\pi^*) - V^*_{r_\theta}(\pi_\theta)\big].
\end{align*}
The third term is non-positive since $\pi_\theta$ maximizes $V^*_{r_\theta}$, and the
first two are each bounded by $L_r\|\theta-\theta^*\|_2$ by the Lipschitz property of $r$.
\end{proof}

\begin{theorem}[Improvement Condition for Mixed Supervision]
\label{thm:mixed_supervision}
Under Assumption~\ref{assump:regularity}, there exists $C>0$ (independent of
$N_o,N_{ai},\epsilon_{ai}$) such that, as $N_o+N_{ai}\to\infty$,
\begin{align*}
\mathrm{Gap}_{\mathrm{oracle}} &\le \frac{C\sqrt d}{\sqrt{N_o}}, \\
\mathrm{Gap}_{\mathrm{mix}} &\le \frac{C\sqrt d}{\sqrt{N_o+N_{ai}(1-2\epsilon_{ai})^2}}.
\end{align*}
Hence the mixed bound is strictly smaller than the oracle-only bound whenever
$\epsilon_{ai}<\tfrac12$, with relative improvement
\begin{align*}
\tfrac{\mathrm{Gap}_{\mathrm{oracle}}}{\mathrm{Gap}_{\mathrm{mix}}}
\;\ge\; \sqrt{1+\tfrac{N_{ai}(1-2\epsilon_{ai})^2}{N_o}}.
\end{align*}
\end{theorem}

\begin{proof}
Combining Lemmas~\ref{lem:rate} and~\ref{lem:policy_gap} with
$C = 2L_r/\sqrt{\lambda_{\min}(\bar I_{\mathrm{clean}})}$ yields both bounds; oracle-only
corresponds to $N_{ai}=0$. The improvement condition reduces to
$N_{ai}(1-2\epsilon_{ai})^2>0$, i.e., $\epsilon_{ai}<\tfrac12$. The ratio follows by
direct computation.
\end{proof}

\section{More Experiments}
\label{app:experiments}

\subsection{Out-of-Domain Generalization}

We evaluate the generalization capability of \framework~by testing models trained on in-domain datasets (\texttt{GSM8K}, \texttt{MATH}, \texttt{WebInstruct}) on challenging out-of-domain benchmarks. We assess performance on three diverse tasks: \textbf{AIME}~\citep{aime_1983_2024}, \textbf{GPQA}~\citep{rein2024gpqa} and \textbf{MMLU-Pro}~\citep{wang2024mmlu}.
Tables~\ref{tab:ood_gsm8k}, \ref{tab:ood_math}, and \ref{tab:ood_webinstruct} report out-of-domain performance for models trained on \texttt{GSM8K}, \texttt{MATH}, \texttt{WebInstruct} respectively. We compare \framework~against four active learning baselines: Random, Entropy, Pref Certainty, and Pref + Ent. 
\begin{table}[h]
\centering
\caption{Out-of-domain performance for models trained on \texttt{GSM8K}.}
\label{tab:ood_gsm8k}
\small
\renewcommand\tabcolsep{4pt}
\renewcommand\arraystretch{1.2}
\begin{tabular}{l|ccc}
\Xhline{1.2pt}
\rowcolor{blue!10}
\textbf{Method} & \textbf{AIME} & \textbf{GPQA} & \textbf{MMLU-Pro} \\
\Xhline{1pt}
Random & 3.33 & 24.12 & 22.45 \\
Entropy & 3.33 & 25.38 & 21.67 \\
Pref Certainty & 2.22 & 27.14 & 24.18 \\
Pref + Ent & 3.33 & 25.89 & 22.76 \\
\Xhline{1pt}
\framework & \textbf{6.67} & \textbf{28.35} & \textbf{25.99} \\
\Xhline{1.2pt}
\end{tabular}
\end{table}

\begin{table}[htbp]
\centering
\caption{Out-of-domain performance for models trained on \texttt{MATH}.}
\label{tab:ood_math}
\small
\renewcommand\tabcolsep{4pt}
\renewcommand\arraystretch{1.2}
\begin{tabular}{l|ccc}
\Xhline{1.2pt}
\rowcolor{blue!10}
\textbf{Method} & \textbf{AIME} & \textbf{GPQA} & \textbf{MMLU-Pro} \\
\Xhline{1pt}
Random & 3.33 & 26.52 & 21.34 \\
Entropy & 0.00 & 27.18 & 22.15 \\
Pref Certainty & 3.33 & 27.48 & 22.08 \\
Pref + Ent & 3.33 & 27.85 & 22.76 \\
\Xhline{1pt}
\framework & \textbf{6.67} & \textbf{28.46} & \textbf{23.42} \\
\Xhline{1.2pt}
\end{tabular}
\end{table}

\begin{table}[htbp]
\centering
\caption{Out-of-domain performance for models trained on \texttt{WebInstruct}.}
\label{tab:ood_webinstruct}
\small
\renewcommand\tabcolsep{4pt}
\renewcommand\arraystretch{1.2}
\begin{tabular}{l|ccc}
\Xhline{1.2pt}
\rowcolor{blue!10}
\textbf{Method} & \textbf{AIME} & \textbf{GPQA} & \textbf{MMLU-Pro} \\
\Xhline{1pt}
Random & 0.00 & 26.83 & 21.54 \\
Entropy & 3.33 & 28.12 & 22.91 \\
Pref Certainty & 3.33 & 27.45 & 22.18 \\
Pref + Ent & 2.22 & 28.76 & 23.45 \\
\Xhline{1pt}
\framework & \textbf{2.22} & \textbf{29.51} & \textbf{24.17} \\
\Xhline{1.2pt}
\end{tabular}
\end{table}

\section{Extended Related Work}
\paragraph{Self-Consistency for LLMs.} Self-consistency was originally proposed as a test-time inference strategy to improve reasoning accuracy~\citep{wang2022self}. The key intuition is that sampling multiple reasoning paths and aggregating answers that appear most frequently provides higher confidence that the consistent answer is correct~\citep{wang2022self, shi-etal-2022-natural, li2022competition}. Recently, researchers have adapted self-consistency from test-time inference to training-time pseudo-labeling for preference learning~\citep{prasad2024self, jiao2025preference,xu-etal-2026-ds2}. The core idea is to generate multiple model responses and use consistency as a signal for quality: responses that consistently arrive at the same answer across different reasoning paths are treated as higher-quality, while inconsistent responses are considered lower-quality. However, existing work using self-consistency for preference learning does not consider how to actively select which queries benefit most from self-consistency-based labeling. Our work addresses this by integrating self-consistency into an active learning framework that selectively applies consistency-based pseudo-labeling to queries where it provides the strongest signal.

\section{Ethics Statement}
We provide comprehensive methodological details to enable reproducibility. Our framework uses exclusively publicly available resources and does not collect or process personal information. ChatGPT, Gemini, and Claude were used solely for minor grammatical and formatting corrections, in accordance with their respective usage policies. While our framework is designed for benign research, we acknowledge potential risks such as bias propagation from automated data synthesis.

\section{Prompts}
This section presents the prompts used throughout our experiments. We organize them into four categories: response generation for constructing preference pairs (green), instruction augmentation with oracle feedback (blue), oracle preference evaluation (red), and zero-shot evaluation on out-of-domain benchmarks (purple).
\label{app:prompts}
\begin{figure*}[t]
\begin{tcolorbox}[notitle, sharp corners, breakable, colframe=ForestGreen, colback=white, 
       boxrule=3pt, boxsep=0.5pt, enhanced, 
       shadow={3pt}{-3pt}{0pt}{opacity=1,mygrey},
       title={Response Generation Prompts}]\label{box:response-gen}
       \scriptsize
       {\fontfamily{pcr}\selectfont
\begin{lstlisting}
# ============================================
# GSM8K Response Generation
# ============================================
gsm8k_response_prompt: str = """
You are given a grade school math word problem involving basic arithmetic, 
algebra, or geometry. Your task is to carefully read the problem and provide 
a step-by-step solution for it.

Provide a step-by-step reasoning process and then write the final numerical 
answer on a new line in the format:
final answer: <answer>
"""

# ============================================
# MATH Response Generation
# ============================================
math_response_prompt: str = """
You are given a competition-level mathematics problem. Your task is to provide 
a detailed step-by-step solution demonstrating rigorous mathematical reasoning.

Present the final result inside a LaTeX boxed expression, i.e., write the answer as \\boxed{<answer>}. 
"""

# ============================================
# WebInstruct Response Generation
# ============================================
webinstruct_response_prompt: str = """
You are given a physics problem that requires numerical reasoning.  Carefully read the problem and provide a step-by-step solution. 
Show all calculations and clearly explain your reasoning. 
Provide a step-by-step solution and write the final answer in the format:
final answer: <answer>
"""
"""
\end{lstlisting}
}
\end{tcolorbox}
\end{figure*}

\begin{figure*}[t]
\begin{tcolorbox}[notitle, sharp corners, breakable, colframe=NavyBlue, colback=white, 
       boxrule=3pt, boxsep=0.5pt, enhanced, 
       shadow={3pt}{-3pt}{0pt}{opacity=1,mygrey},
       title={Question Augmentation with Oracle Feedback}]\label{box:augmentation}
       \scriptsize
       {\fontfamily{pcr}\selectfont
\begin{lstlisting}
# ============================================
# GSM8K Instruction Augmentation
# ============================================
gsm8k_augmentation_prompt: str = """
Based on the examples above, generate ONE solvable math word problem with 
similar difficulty. Ensure all information needed to solve the problem is 
included in the question.

Output the question and nothing else.
Q:
"""
# ============================================
# MATH Instruction Augmentation
# ============================================
math_augmentation_prompt: str = """
Based on the examples above, generate ONE challenging mathematics problem 
with similar difficulty and topic. Ensure the problem is well-defined and 
solvable with the given information.

Output the question and nothing else.
Q:
"""

# ============================================
# WebInstruct Instruction Augmentation
# ============================================
webinstruct_augmentation_prompt: str = """
Based on the examples above, generate ONE solvable physics problem with 
similar difficulty and topic. The question should require numerical reasoning and may involve units or currency. Ensure all information needed to solve the problem is included.

Output the question and nothing else.
Q:
"""
\end{lstlisting}
}
\end{tcolorbox}
\label{fig:augmentation_prompts}
\end{figure*}

\begin{figure*}[t]
\begin{tcolorbox}[notitle, sharp corners, breakable, colframe=BrickRed, colback=white, 
       boxrule=3pt, boxsep=0.5pt, enhanced, 
       shadow={3pt}{-3pt}{0pt}{opacity=1,mygrey},
       title={Oracle Preference Evaluation}]\label{box:oracle}
       \scriptsize
       {\fontfamily{pcr}\selectfont
\begin{lstlisting}
# ============================================
# Oracle Evaluation Prompt
# ============================================
oracle_evaluation_prompt: str = """
You are an expert evaluator for mathematical/pysical reasoning problems. Evaluate two 
responses to a math question and output your evaluation as a JSON object.

Question: {question}

Response 1: {full_response1}

Response 2: {full_response2}

Given the ground truth full response: {ground_truth_full_response}

Evaluate the responses following this logic:
1. Check if Response 1's final answer is correct
2. Check if Response 2's final answer is correct
3. Determine preference based on correctness:
   - If only one response is correct, prefer the correct one
   - If both responses are correct, prefer the one with better reasoning/explanation
   - If both responses are incorrect, prefer the one with better reasoning/explanation

Output your evaluation as a JSON object with the following structure:
{{
    "response1_correct": true/false,
    "response2_correct": true/false,
    "response1_preferred": true/false,
    "reasoning": "Brief explanation of your evaluation"
}}

Only output the JSON object, no additional text.
"""
\end{lstlisting}
}
\end{tcolorbox}
\end{figure*}

\begin{figure*}[t]
\begin{tcolorbox}[notitle, sharp corners, breakable, colframe=purple, colback=white, 
       boxrule=3pt, boxsep=0.5pt, enhanced, 
       shadow={3pt}{-3pt}{0pt}{opacity=1,mygrey},
       title={Zero-Shot Evaluation Prompts}]
\label{box:zeroshot}
\scriptsize
{\fontfamily{pcr}\selectfont
\begin{lstlisting}
# ============================================
# AIME Zero-Shot Prompt
# ============================================
aime_zeroshot_prompt: str = """
You are given an American Invitational Mathematics Examination (AIME) problem. 
These are challenging olympiad-level problems requiring creative mathematical 
thinking. Provide a rigorous solution with clear mathematical reasoning.

Provide a step-by-step solution and write the final answer in the format:
final answer: <answer>
"""

# ============================================
# GPQA Zero-Shot Prompt
# ============================================
gpqa_zeroshot_prompt: str = """
You are given a graduate-level multiple choice question from physics, chemistry, or biology. Analyze each option carefully based on established scientific principles.

Provide a step-by-step explanation and write the final answer in the format:
final answer: <A/B/C/D>
"""

# ============================================
# MMLU-Pro Zero-Shot Prompt
# ============================================
mmlu_pro_zeroshot_prompt: str = """
You are given a multiple-choice question that tests knowledge across various 
domains. Analyze the question carefully, consider each option, and provide 
your reasoning before selecting the best answer.

Provide a step-by-step explanation and write the final answer in the format:
final answer: <A/B/C/D>
"""
\end{lstlisting}
}
\end{tcolorbox}
\label{fig:zeroshot_prompts}
\end{figure*}

\end{document}